\documentclass[10pt,twocolumn,letterpaper]{article}

\usepackage{cvpr}              

\definecolor{cvprblue}{rgb}{0.21,0.49,0.74}
\usepackage[pagebackref,breaklinks,colorlinks,allcolors=cvprblue]{hyperref}
\usepackage{amsmath}
\usepackage{amssymb}
\usepackage{mathtools}
\usepackage{amsthm}
\usepackage{bbm}
\usepackage{algpseudocodex}
\usepackage{algorithm}
\usepackage{tabularx}
\newtheoremstyle{mytheoremstyle} 
    {\topsep}                    
    {\topsep}                    
    {}                   
    {}                           
    {\bfseries}                   
    {}                          
    {.5em}                       
    {}  

\theoremstyle{mytheoremstyle}

\newtheorem{theorem}{Theorem}[section]

\newtheorem{lemma}[theorem]{Lemma}

\newtheorem{assumption}[theorem]{Assumption}

\usepackage{thmtools}
\usepackage{thm-restate}

\def\paperID{14726} 
\def\confName{CVPR}
\def\confYear{2025}

\title{Stop learning it all to mitigate visual hallucination, \\
Focus on the hallucination target.}

\author{Dokyoon Yoon, \qquad Youngsook Song, \qquad Woomyong Park\textsuperscript{*}\\
SIONIC AI\\
{\tt\small dokyoon, song, max@sionic.ai}
}

\newcommand{\mymethod}{TL-DPO}
\numberwithin{equation}{section}

\begin{document}
\maketitle
\begin{abstract}
Multimodal Large Language Models (MLLMs) frequently suffer from hallucination issues, generating information about objects that are not present in input images during vision-language tasks. These hallucinations particularly undermine model reliability in practical applications requiring accurate object identification. To address this challenge, we propose \mymethod,\ a preference learning approach that mitigates hallucinations by focusing on targeted areas where they occur. To implement this, we build a dataset containing hallucinated responses, correct responses, and target information (i.e., objects present in the images and the corresponding chunk positions in responses affected by hallucinations). By applying a preference learning method restricted to these specific targets, the model can filter out irrelevant signals and focus on correcting hallucinations. This allows the model to produce more factual responses by concentrating solely on relevant information. Experimental results demonstrate that \mymethod\ effectively reduces hallucinations across multiple vision hallucination tasks, improving the reliability and performance of MLLMs without diminishing overall performance.
\end{abstract}    

\renewcommand{\thefootnote}{}
\footnotetext{\textsuperscript{*}Corresponding author}
\renewcommand{\thefootnote}{\arabic{footnote}} 

\section{Introduction}
\label{sec:intro}

\begin{figure}[t]
  \centering
   \includegraphics[width=0.8\linewidth]{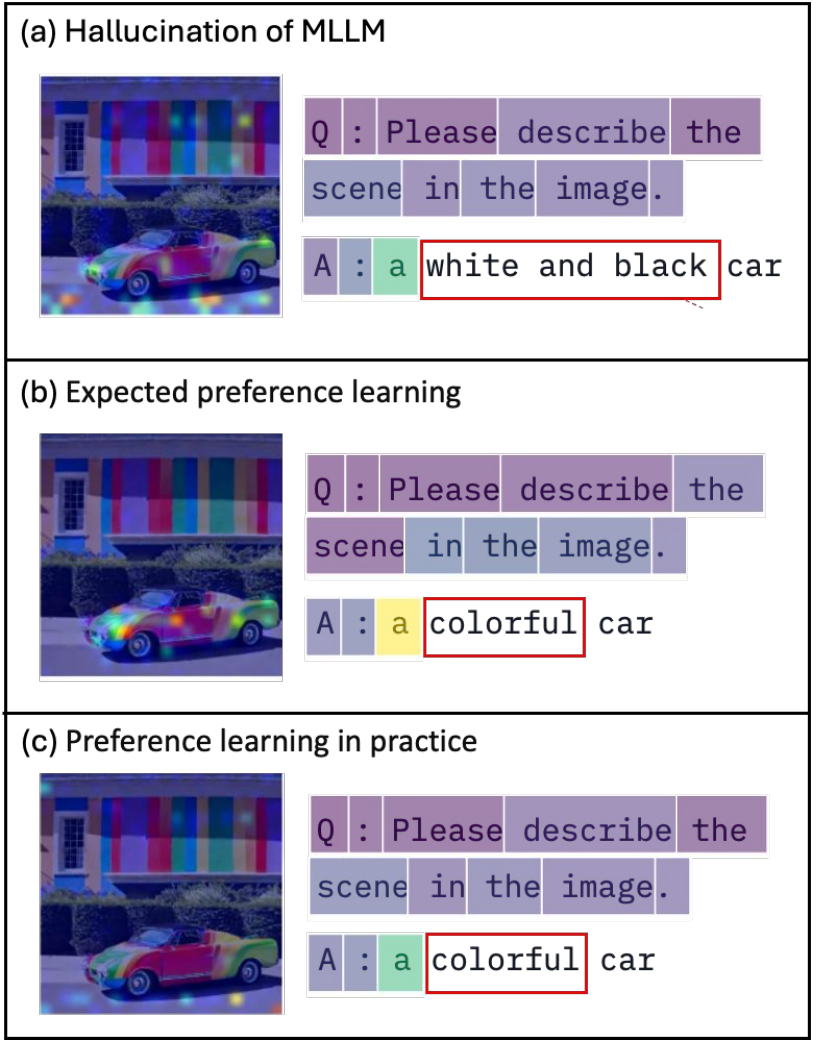}
   \caption{\textbf{Visualization of attention maps for hallucination mitigation} shows how preference learning influences model's attention. (a) illustrates hallucinations caused by focusing on incorrect objects or missing relevant ones. Ideally, as in (b), preference learning would guide attention to the correct objects, but in reality, there is a risk of overfitting to text rather than object-based information, as shown in (c). Additional analysis is in Appendix A.
    }
   \label{fig:attention}
\end{figure}

Multimodal Large Language Models (MLLMs) have achieved remarkable performance in multimodal tasks by leveraging extensive training and fine-tuning on large-scale image-text pairs, which has significantly enhanced their contextual understanding and instruction-following abilities \cite{instructblip,llava,gpt4o, docowl}. However, during the integration of text and vision modalities, a new type of hallucination issue has emerged, where the model generates object information that does not exist in the images or creates inaccurate spatial representations \cite{mllm-ha-survey, preventinghallucinations, evaluatingobjecthallucination, mitigatinghallucination, hallecontrol}.

To address these issues, recent studies have introduced preference-based learning methods, most notably Reinforcement Learning from Human Feedback (RLHF), to improve alignment across modalities and correct hallucinations \cite{rlhfv, rlaifv, mdpo, multilevelpreference, dpa}. However, the problem of object-level hallucinations within images remains unresolved \cite{objecthalluci, hallecontrol}. Although preference-based learning has contributed to reducing distribution gaps across modalities and enhancing factuality, its global feature-centric learning poses limitations in addressing object-level hallucinations. 

\begin{figure*}
    \centering
    \includegraphics[width=0.9\linewidth]{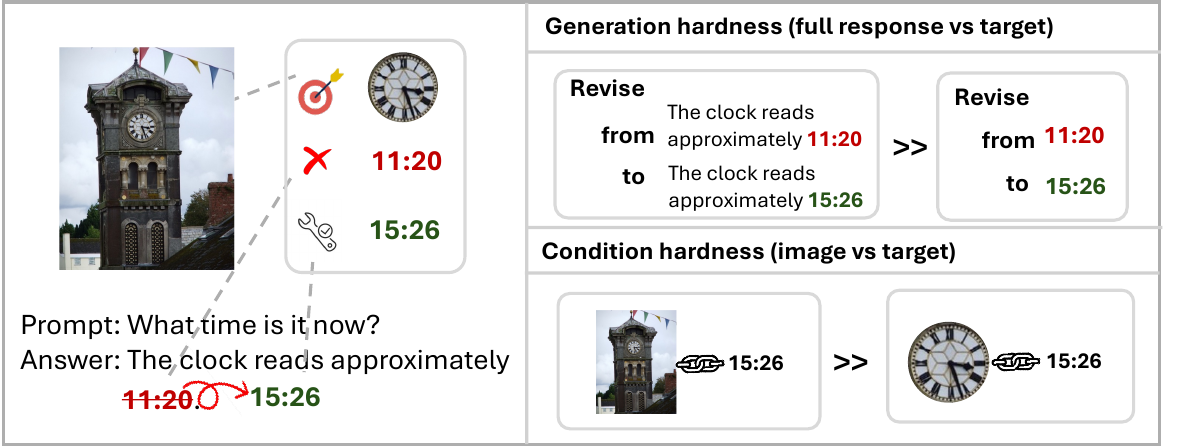}
    \caption{Target learning paradigm. (Left) Correcting hallucinations in image-based QA requires identifying the hallucinated object, the hallucination, and the correct (e.g., identifying incorrect and correct times for a clock-related hallucination). (Right) Training on specific hallucinated parts and their corrections is easier than learning from whole responses, and focusing on relevant image regions simplifies learning compared to using the entire image. 
    }
    \label{fig:overview}
\end{figure*}

While humans tend to correct only specific areas when they misdraw a region in an illustration or use an incorrect word in writing, previous approaches to addressing hallucination issues in the literature have often suggested ``redrawing everything." In this study, we examine the necessity of directly addressing the image hallucination problem, as humans would, by investigating whether existing preference-based learning aligns with our expectations for hallucination correction. As shown in \ref{fig:attention}, existing preference learning methods have often learned signals unrelated to hallucinations (e.g., irrelevant chunks or regions in the image), leading to unintended side effects. 

Based on these observations, we hypothesize that the reason for the lack of effective hallucination correction in preference learning is due to learning from irrelevant signals. Consequently, we propose a novel approach, termed “target learning”, which focuses exclusively on the image regions where hallucinations occur and the specific chunks in the responses that contain hallucinations, separating these as the primary targets for learning. Additionally, we provide a logical explanation showing how target learning enhances both the efficiency and effectiveness of the learning process.

To correct hallucinations in MLLMs, this study introduces \textbf{\mymethod} (\textbf{T}arget-\textbf{L}earning \textbf{D}irect \textbf{P}reference \textbf{O}ptimization), a learning method that identifies image regions and response chunks associated with hallucinations as targets, while categorizing unrelated parts as irrelevant signals, thereby focusing solely on the target elements during training. In \mymethod, a new loss function is employed to focus on targets, enabling selective learning by distinguishing them from irrelevant signals. This approach not only guides multimodal models to generate outputs without hallucinations but also mitigates other incorrect behaviors.

To sum up, our contributions are  as follows:
\begin{enumerate}

\item We identify the learning of irrelevant signals as the primary reason for the failure of conventional preference learning in correcting hallucinations. To address this, we propose \textbf{target learning}, which explicitly excludes irrelevant signals, and provide a solid theoretical foundation for this approach.

\item We introduce \mymethod \ a novel method for implementing Target Learning. To support this method, we expand the existing dataset by constructing a preference-based image dataset that includes the target 

\item We validate the effectiveness of the proposed method through experiments on various benchmark models, demonstrating reduced hallucinations.

\end{enumerate}
\section{Preliminaries}

In language generation, a language model (LM) generates a response $y$ based on a prompt containing a question $x$, where both $x$ and $y$ consist of sequences of tokens. Direct Preference Optimization (DPO) \cite{dpo} builds upon the reinforcement learning (RL) objective used in RLHF (Reinforcement Learning from Human Feedback) \cite{rlhf}.

\begin{equation}
\begin{aligned}
\max_{\pi_{\theta}} \mathbb{E}_{x \sim D, y \sim \pi_{\theta}(\cdot \mid x)} [ &r(x, y) \\ - \beta D_{\text{KL}} ( &\pi_{\theta}(\cdot \mid x) | \pi_{\text{ref}}(\cdot \mid x) )  ],
\label{eq:1}
\end{aligned}
\end{equation}
where $D$ represents the human preference dataset, $r(x, y)$ denotes the reward function, and $\pi_{\text{ref}}(\cdot \mid x)$ is the reference model, typically chosen as the language model after supervised fine-tuning. Here, $\pi_{\theta}$ denotes the model being adjusted through reinforcement learning. This objective function incorporates Kullback-Leibler (KL) divergence constraints at the token level, improving the regulation of KL divergence.

DPO derives a connection between the reward model and the optimal policy by applying reverse KL divergence, allowing the reward function to be expressed in relation to the model's policy. This is captured as:

\begin{equation}
r(x, y) = \beta \log \frac{\pi_{\theta}(y \mid x)}{\pi_{\text{ref}}(y \mid x)} + \beta \log Z(x),
\label{eq:BT}
\end{equation}
where $Z(x)$ is a normalization term called the partition function. To effectively capture human preferences, DPO uses the Bradley-Terry model to compare pairs of possible responses, assigning a probability to each preferred response $y_1$ over another response $y_2$ given the prompt $x$:

\begin{equation}
P_{\text{BT}}(y_1 \succ y_2 \mid x) = \frac{\exp(r(x, y_1))}{\exp(r(x, y_1)) + \exp(r(x, y_2))}.
\end{equation}

This formulation is refined by substituting the reward expression into the comparison model, leading to a loss function based on the negative log-likelihood:

\begin{equation}
\mathcal L_{\text{DPO}}(\pi_{\theta}; \pi_{\text{ref}}) = -\mathbb{E}_{(x, y_w, y_l) \sim D} [\log \sigma (u(x, y_w, y_l))],
\end{equation}
where $\sigma(\cdot)$ is the sigmoid function and $u(x, y_w, y_l)$ quantifies the relative preference between the more favorable response $y_w$ and the less favorable response $y_l$. Specifically, $u(x, y_w, y_l)$ is defined as:

\begin{equation}
u(x, y_w, y_l) = r(x, y_w) - r(x, y_l).
\label{eq:u}
\end{equation}

The gradient of this loss, critical for model updates, is:
\begin{equation}
\begin{array}{l}
\nabla_{\theta} \mathcal L_{\text{DPO}}(\pi_{\theta}; \pi_{\text{ref}}) = \\ 
 -\mathbb{E}_{(x, y_w, y_l) \sim D} \left[\sigma (-u(x, y_w, y_l)) \nabla_{\theta} u \right],
\end{array}
\end{equation}
where $u$ acts as a shorthand for $u(x, y_w, y_l)$. This formulation enables DPO to iteratively refine the policy by focusing on minimizing the divergence while aligning with human preference through straightforward pairwise comparison, enhancing the overall quality of the model's responses without complex reward model sampling.





\section{Target Learning Problem}
\label{sec:problem}

Although conventional preference learning has been effective in improving the overall output quality of models, as shown in \ref{fig:overview}, it may also result in learning elements that are not designated as targets in preference learning. This is due to a limitation, as represented in \ref{fig:attention}, where the optimization process prioritizes global features, often overlooking chunks where hallucinations occur in responses or localized object information that has caused hallucinations.

In \ref{sec:target learning}, we redefine conventional preference learning as a problem focused solely on target learning. We formalize this as a target learning approach and logically demonstrate in \ref{sec:bt-model} that the target learning approach achieves greater efficiency than conventional preference learning by excluding irrelevant signals.

This approach rejects hallucinated values in target chunks where hallucinations occur, guiding the model to favor corrected values. In the human preference learning stage, a reward model $\hat{r}$ is trained on target chunks, enabling the model to output hallucination-free chunks and reducing the probability of generating hallucinated chunks. This reward model can assign rewards across various chunks, reflecting the presence of hallucinations as human preference. Once the reward model is established, it provides feedback in an additional fine-tuning stage, guiding a policy model $\pi_\theta$ based on human preferences regarding the presence or absence of hallucinations.



\subsection{Target optimization} 
\label{sec:target learning}
In this section, we first reformulate the RLHF \cite{rlhf} objective (constrained reward maximization problem) in terms of the target. The objective function of DPO in Eq. \ref{eq:1} operates at the sentence level. In contrast, we propose a new objective function designed to work on alternative target generation in the context of image goals:

\begin{equation}
\begin{aligned}
\max_{\pi_{\theta}} \mathbb{E}_{x \sim D, y^t \sim \pi_{\theta}(\cdot \mid x)} [ &r(x, y^t) \\ - \beta D_{\text{KL}} ( &\pi_{\theta}(\cdot \mid x) | \pi_{\text{ref}}(\cdot \mid x) )  ],
\label{eq:1}
\end{aligned}
\end{equation}
where $y^t$ represents the subset of targets corresponding to the full target set $T$. This objective function is designed to enable the model to learn more effectively toward specified targets in image generation, while regulating the KL divergence with the reference model’s policy. This balance enhances the model’s stability and maintains the quality of generated outputs.

In this work, we design the MLLMs to enable preference learning by comparing chunks within the response where hallucinations occur and the localized objects that cause them, to correct hallucinations. For this, we set the assumption that human preference differences regarding hallucinations occur only in signals related to hallucinations, with no differences in unrelated signals like below:

\begin{assumption}\label{assum}
For hallucination-related response $y$, the reward for response satisfies $\sum r_t = \sum r$, where $r_t$ represents the reward funtion for the target in response $y^t$.
\end{assumption}

Following, in the preference learning problem dealing with hallucinations, we can see that the value of the Bradley-Terry model remains the same when unrelated signals are removed:

\begin{lemma}
Given a reward function $r(x, y)$, assuming Assumption \ref{assum}, we can establish the equivalence between the Bradley-Terry model to the \ref{eq:BT}.
\label{lemma}
\end{lemma}

Lemma \ref{lemma} ensures that, under Assumption \ref{assum}, the preference reward for a response remains the same even when non-target tokens are excluded, preserving the overall reward  $r(x, y)$. Here, $r(x, y)$  represents the overall reward for the response  $y$  given the prompt  $x$. Therefore, when considering text generation as a target-level generation problem,  $r(x, y^t)$  can be viewed as the reward for the generated text aligned with the target, rather than the entire sentence.

Using Lemma \ref{lemma}, we can limit the calculation scope to the target itself. This strategy can also be applied in DPO by restricting the comparison of text generations to the target. This concept is summarized in the following theorem.
\begin{restatable}[]{theorem}{theoremonetext}\label{equi}
Under Assumption \ref{assum}, preference learning remains equivalent when non-target tokens are excluded. 
\end{restatable}

That is, the reward function and policy updates for target-restricted DPO align with those of the original DPO. Formally, let $\pi_\theta$ denote the policy in DPO and  $r(x, y)$  the reward function. When Assumption \ref{assum} holds, for any prompt  $x$  and response  $y$:

\begin{equation}
\begin{aligned}
\mathbb{E}_{y \sim \pi_\theta} \left[ r(x, y) \right] = \mathbb{E}_{y^t \sim \pi_\theta} \left[ r(x, y^t) \right],
\end{aligned}
\end{equation}

Thus, the policy gradient update for target-restricted DPO follows the same form as for the original DPO:

\begin{equation}
\begin{aligned}
\nabla_\theta \mathbb{E}_{y^t \sim \pi_\theta} \left[ r(x, y^t) - \beta D_{\text{KL}} (\pi_\theta | \pi_\text{ref}) \right]
\end{aligned}
\end{equation}

The proof of theorem \ref{equi} is provided in the Appendix B.
Theorem \ref{hf} provides a structured approach for modeling human preference probability within a KL-constrained setting by using the Bradley-Terry model. In this framework, the optimal policy  $\pi^{}_{\theta}$  is designed to align closely with human preferences by directly modeling the relative preference between two target outputs,  $y_h^t$  and  $y_r^t$ , conditioned on the prompt  $x$ .

\subsection{Efficiency in Target Learning} 
\label{sec:bt-model}
In \ref{sec:target learning}, we were able to modify the existing Bradley-Terry model to the target level, transforming the preference objective (eq.\ref{eq:1}) into a target-level problem. 	Intuitively, as shown in \ref{fig:overview}, focusing on correcting the target rather than the entire sentence is more efficient. This approach enables the model to learn the relationship between the specific region in the image where the fact is located and the fact itself, rather than the relationship between the entire image and the fact. This section emphasizes that the proposed target learning approach is theoretically more efficient than traditional preference learning methods that focus on the overall output. Traditional preference learning can indeed enhance the general quality of model outputs; however, it also tends to include elements outside of the designated target. This inefficiency arises as non-target elements in the image are often learned, resulting in a need for more extensive data during the learning process.

\newtheorem{prop}{Proposition}
\begin{prop} \textbf{Efficiency comparison in target learning} 
Let $\mathcal{H}_{pl}$ and $\mathcal{H}_{tl}$ be the hypothesis spaces associated with conventional preference learning (pl) and target-focused preference learning (tl), respectively. Let $m_{pl}$ and $m_{tl}$ be the number of samples required by each respective method to achieve a generalization error $\epsilon$ with confidence level $1 - \delta$. Then, $m_{tl} < m_{pl}$ holds.
\end{prop}

By proving the above proposition in the Appendix B, it is shown that the target learning approach can achieve the same level of generalization performance with fewer samples compared to conventional holistic learning methods. This improvement in efficiency is expected to yield a more effective hallucination mitigation effect in real.

\begin{figure*}
\centering
\includegraphics[width=400pt]{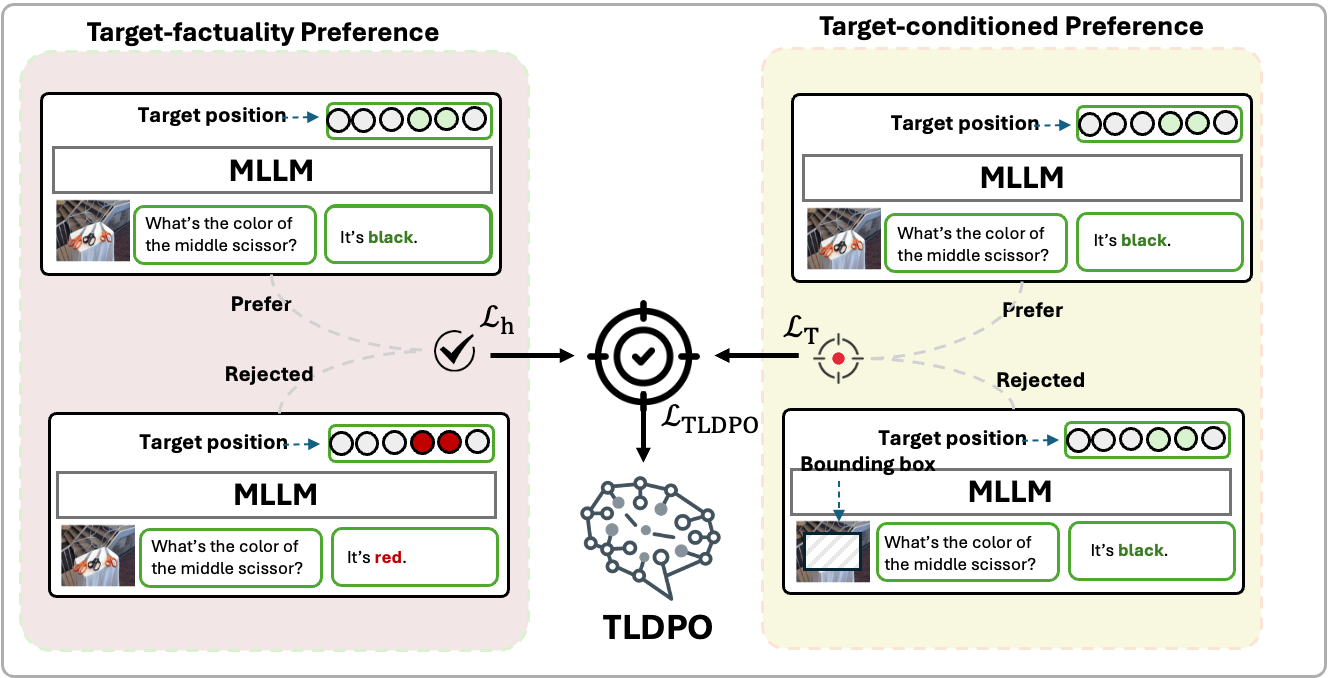}
\caption{Overview of TL-DPO. To correct hallucinations, TL-DPO (right) applies learning only to hallucinated target chunks, correcting only erroneous information and learning factual information. (left) It applies a masked condition to the target object where hallucination occurs utilizing bounding box information from the image dataset, training the model to extract accurate information from the desired part of the image. TL-DPO combines these two learning objectives to correct erroneous responses extracted from the desired image regions and transform them into correct responses.(Details in Appendix D)}
\label{fig:topic_clustering}
\end{figure*}

\section{Methods}

\label{sec:method}
To address the tendency of conventional Preference Learning to ignore images and learn from undesired signals, this study proposes a new learning approach in \ref{sec:TL-Dataset} and \ref{sec:TL-DPO} to mitigate these issues. The primary objective is to resolve the object hallucination problem in MLLMs through targeted learning. Considering the complexity of data construction and model training, we have chosen Direct Preference Optimization (DPO) over traditional RLHF (RL with Human Feedback) methods to restrict model preferences. DPO, a simpler and more efficient approach that does not require reinforcement learning, serves as our fundamental strategy. We extend this strategy to the multimodal domain to eliminate hallucinations and enhance the reliability and accuracy of multimodal model outputs.
We construct a TL dataset that includes incorrect responses, correct responses, and the positional information of targets with bounding boxes (Details in Appendix C). This dataset is subsequently used to train the \mymethod\  model. Models trained using this method show a substantial reduction in hallucinations in tasks involving detailed descriptions and dialogue.

\subsection{Generation of Target-Object-Inclusive Preference Datasets}
\label{sec:TL-Dataset}
In this study, we utilize the Visual Genome (VG) dataset as the main source to create hallucinated and corrected samples. For each given image and related question, we use a baseline model to generate an initial response, then identify instances where hallucinations occur to produce the hallucinated response  $y_h$  and its corrected version  $y_r$. Throughout this process, we record the positions of hallucinated elements in the original response and the corresponding corrected elements, allowing us to define hallucinated segments  $y_h^t$  and their corrected counterparts  $y_r^t$. Additionally, we incorporate bounding box information from the VG dataset to identify the specific objects within the image that contribute to the hallucination, resulting in the hallucination-related object information  $m_i^t$. Consequently, the structured dataset contains the hallucinated response  $y_h$  along with its hallucinated segment  $y_h^t$, the corrected response  $y_r$  with its corrected segment  $y_r^t$, the full image  $m_i$, and the specific hallucination-triggering object within the image  $m_i^t$. This comprehensive setup facilitates detailed analysis and targeted correction of hallucinations in generated responses. Detailed explanations are provided in Appendix C.

\subsection{Proposed Loss in \mymethod}
\label{sec:TL-DPO} 
\paragraph{Target Generation Loss}
\label{target loss}
When learning from hallucinations, the model may unintentionally learn irrelevant information, which we address by applying target loss. Given a pretrained MLLM that is prone to hallucinations, our goal is to minimize the likelihood of generating hallucinated responses by using pairs of erroneous and corrected responses, ${y_h^t, y_r^t}$, obtained through generative data augmentation. To this end, we define a preference optimization expression for erroneous and corrected response pairs.

For example, as shown in \ref{fig:overview}, suppose the erroneous response $y_h$  is "The clock reads approximately 11:20", and the corrected response  $y_r$  is "The clock reads approximately 15:26". Here,  $y_h^t$  represents the hallucinated information in  $y_h$ , and  $y_r^t$  denotes the corrected information in  $y_r$  that replaces  $y_h^t$  with the correct answer. In this example, the erroneous information  $y_h^t$  is "11:20", and the correct information  $y_r^t$  is "15:26".

Assuming that irrelevant information does not affect the Reward difference, the DPO formula\ref{eq:1} can be formulated as follows :

\begin{equation}
\mathcal L_{\text{t}} = -\mathbb{E}_{(x, y_r^t, y_h^t) \sim D} [\log \sigma (u(x, y_r^t, y_h^t))],
\label{eq:to}
\end{equation}
\begin{table*}[ht]
\centering
\begin{tabular}{|l|l|l|l|l|l|l|l|l|}

\hline
                  & \multicolumn{4}{l|}{Hallucination Benchmark}                                                       & \multicolumn{4}{l|}{Comprehensive Benchmark}                                                                                                           \\ \hline
Method            & \multicolumn{1}{l|}{\begin{tabular}[c]{@{}l@{}}CHAIR\_s\\ ($\downarrow$)\end{tabular}} & \multicolumn{1}{l|}{\begin{tabular}[c]{@{}l@{}}CHAIR\_i\\ ($\downarrow$)\end{tabular}} & \multicolumn{1}{l|}{\begin{tabular}[c]{@{}l@{}}POPE\\ ($\uparrow$)\end{tabular}}  & {\begin{tabular}[c]{@{}l@{}}MMHal\\ ($\uparrow$)\end{tabular}} & \multicolumn{1}{l|}{\begin{tabular}[c]{@{}l@{}}SciQA-\\ IMG\end{tabular}($\uparrow$)} & \multicolumn{1}{l|}{{\begin{tabular}[c]{@{}l@{}}MM-Vet\\ ($\uparrow$)\end{tabular}}} & \multicolumn{1}{l|}{{\begin{tabular}[c]{@{}l@{}}MMBench\\ ($\uparrow$)\end{tabular}}} & \begin{tabular}[c]{@{}l@{}}LLaVA-\\ Bench($\uparrow$)\end{tabular}  \\ \hline
LLaVA-1.5         & \multicolumn{1}{l|}{66.8}     & \multicolumn{1}{l|}{12.7}     & \multicolumn{1}{l|}{85.9}  & 2.42  & \multicolumn{1}{l|}{66.8}                                                 & \multicolumn{1}{l|}{30.5}   & \multicolumn{1}{l|}{63.0}    & 63.4          \\ \cline{2-9} 
+VLfeedback       & \multicolumn{1}{l|}{56.3}     & \multicolumn{1}{l|}{11.4}     & \multicolumn{1}{l|}{83.72} & 2.62  & \multicolumn{1}{l|}{66.2}                                                 & \multicolumn{1}{l|}{31.2}   & \multicolumn{1}{l|}{63.9}    & 62.1          \\ \cline{2-9} 
+Human-Preference & \multicolumn{1}{l|}{54.0}     & \multicolumn{1}{l|}{9.3}      & \multicolumn{1}{l|}{81.50} & 2.53  & \multicolumn{1}{l|}{65.8}                                                 & \multicolumn{1}{l|}{31.1}   & \multicolumn{1}{l|}{60.4}    & 63.7          \\ \cline{2-9} 
+RLHF-V           & \multicolumn{1}{l|}{44.6}     & \multicolumn{1}{l|}{7.9}      & \multicolumn{1}{l|}{86.20} & 2.59  & \multicolumn{1}{l|}{67.1}                                                 & \multicolumn{1}{l|}{30.9}   & \multicolumn{1}{l|}{63.6}    & 65.4  \\ \cline{2-9} 
+POVID & 35.2 & 8.3 & 86.9 & 2.08 & 67.5&30.5 & 63.2 &64.8 \\ \cline{2-9} 
+HA-DPO  & 37.2 & 10.0 & 86.9 & 1.97&70.3 &30.7 & 64.0 & 66.2\\ \cline{2-9} 
+HALVA & 46.6 & 23.1 & \textbf{87.0} & 2.25& \textbf{70.5} &32.1 & 66.1 &67.2 \\ \hline
\textbf{+ TL-DPO}     & \multicolumn{1}{l|}{\textbf{20.1}}     & \multicolumn{1}{l|}{\textbf{5.2}}      & \multicolumn{1}{l|}{86.95} & \textbf{2.72}  & \multicolumn{1}{l|}{70.3}                                                 & \multicolumn{1}{l|}{\textbf{32.2}}   & \multicolumn{1}{l|}{\textbf{67.8}}    & \textbf{71.2} \\ \hline
\end{tabular}
\caption{\textbf{Main results on the Hallucination Benchmark and Comprehensive Benchmark.} We report CHAIR-s, CHAIR-i, POPE, and MMHal scores under the Hallucination Benchmark, and SciQA-IMG, MM-Vet, MMBench, and LLaVA-Bench scores under the Comprehensive Benchmark. The best results for each metric in each group are presented in bold. Additionally, we provide additional results using various preference data and learning objectives for reference.}
\label{tab:general}
\vspace{-10pt}
\end{table*}
\begin{table*}[]
\centering
\begin{tabular}{|l|l|l|l|l|l|l|l|l|}
\hline
                  & \multicolumn{4}{l|}{Hallucination Benchmark}                                                                                 & \multicolumn{4}{l|}{Comprehensive Benchmark}                                                                                                                                                                 \\ \hline
Method            & \multicolumn{1}{l|}{\begin{tabular}[c]{@{}l@{}}CHAIR\_s\\ ($\downarrow$)\end{tabular}} & \multicolumn{1}{l|}{\begin{tabular}[c]{@{}l@{}}CHAIR\_i\\ ($\downarrow$)\end{tabular}} & \multicolumn{1}{l|}{\begin{tabular}[c]{@{}l@{}}POPE\\ ($\uparrow$)\end{tabular}}  & {\begin{tabular}[c]{@{}l@{}}MMHal\\ ($\uparrow$)\end{tabular}} & \multicolumn{1}{l|}{\begin{tabular}[c]{@{}l@{}}SciQA-\\ IMG\end{tabular}($\uparrow$)} & \multicolumn{1}{l|}{{\begin{tabular}[c]{@{}l@{}}MM-Vet\\ ($\uparrow$)\end{tabular}}} & \multicolumn{1}{l|}{{\begin{tabular}[c]{@{}l@{}}MMBench\\ ($\uparrow$)\end{tabular}}} & \begin{tabular}[c]{@{}l@{}}LLaVA-\\ Bench($\uparrow$)\end{tabular}  \\ \hline
LLaVA-1.5         & \multicolumn{1}{l|}{66.8}          & \multicolumn{1}{l|}{12.7}         & \multicolumn{1}{l|}{85.90}          & 2.42          & \multicolumn{1}{l|}{66.8}                                                 & \multicolumn{1}{l|}{30.5}          & \multicolumn{1}{l|}{63.0}          & 63.4                                                   \\ \cline{2-9} 
\textbf{+ TL-DPO} & \multicolumn{1}{l|}{\textbf{20.1}} & \multicolumn{1}{l|}{\textbf{5.2}} & \multicolumn{1}{l|}{\textbf{86.95}} & \textbf{2.72 }         & \multicolumn{1}{l|}{\textbf{70.3}}                                                 & \multicolumn{1}{l|}{\textbf{32.2}} & \multicolumn{1}{l|}{\textbf{67.8}} & \textbf{71.2}                                          \\ \hline
Qwen VL Chat      & \multicolumn{1}{l|}{48.2}          & \multicolumn{1}{l|}{9.1}          & \multicolumn{1}{l|}{87.07}          & 2.89          & \multicolumn{1}{l|}{\textbf{68.2}}                                                 & \multicolumn{1}{l|}{\textbf{41.2}}          & \multicolumn{1}{l|}{\textbf{60.6}}          &\textbf{67.6 }                                                  \\ \cline{2-9} 
\textbf{+ TL-DPO} & \multicolumn{1}{l|}{\textbf{31.4}} & \multicolumn{1}{l|}{\textbf{7.8}} & \multicolumn{1}{l|}{\textbf{85.28}} & \textbf{2.89} & \multicolumn{1}{l|}{64.8}                                        & \multicolumn{1}{l|}{38.5}          & \multicolumn{1}{l|}{59.7}          & 64.0                                                   \\ \hline
LLaVA-Next & \multicolumn{1}{l|}{29.08} & \multicolumn{1}{l|}{8.08} & \multicolumn{1}{l|}{84.8} & \multicolumn{1}{l|}{2.56} & 70.3 & \textbf{40.2 }& 63 & \textbf{72.7} \\\cline{2-9} 
\textbf{+ TL-DPO} & \textbf{25.1} & \textbf{6.72} & \textbf{87.1} & \textbf{3.12} & \textbf{72.1} & 40.1 & \textbf{63.1} & 72.6 \\
\hline
Llama3 & \textbf{5.5} & 10.1 & 82.8 & 3.19 & 83.9 & \textbf{57.6} & 85.8 & \textbf{65.8} \\\cline{2-9} 
\textbf{+ TL-DPO} & 7.1 & \textbf{9.71} & \textbf{87.1} & \textbf{3.86} &\textbf{87.2} & 53.8 & \textbf{87.3 }& 58.5 \\
\hline
InternVL-2.5(8B) & 18.4 & 8.7 & 86.5 & 3.3 & 96.3& 62.8& 68.6& 82.5 \\\cline{2-9} 
\textbf{+ TL-DPO} & \textbf{7.6} & \textbf{4.1} & \textbf{87.0} & \textbf{3.6} & \textbf{97.1 }& \textbf{68.1} & \textbf{80.0} & \textbf{87.2} \\
\hline
\end{tabular}
\caption{\textbf{Comparative results of different MLLMs with and without \mymethod. } Results are provided for baseline models (LLaVA-1.5, DocOwl, Qwen VL Chat, LLaVA-Next, Llama3, InternLV-2.5) both with and without TL-DPO for comparison, with the improved scores highlighted in bold.}
\label{tab:multi}
\vspace{-15pt}
\end{table*}

Here, $t$ denotes a target, $y_r^t$ is a target from the corrected response, and $y_h^t$ is a target from the hallucinated response. $T$ represents the set of targets and $u$ is defined in eq.\ref{eq:u}. In summary, while the standard DPO objective maximizes $\sigma (r(x, y_r) - r(x, y_h))$, \mymethod\, maximizes the target-specific reward $\sigma(r(x, y_r^t) - r(x, y_h^t))$, leveraging the fact that not all positions reflect the reward in \mymethod.

\paragraph{Target Condition Loss}
\label{sec:condition}
In this study, we propose target-conditional preference optimization to guide the model to utilize target object information. The key idea is to extend the approach proposed in mDPO\cite{mdpo} by leaving the target object in the image as a variable, allowing the model to determine preference labels based solely on target object information.
For example, as shown in \ref{fig:overview}, let's assume that the given image $m_i$ is of a clock tower, and $m_i^t$ represents the specific part of the image that caused the hallucination $y_h^t$. In $~m_i^t$, the object in the image that caused the hallucination is masked. In this example, $y_r^t$ is "11:20", $m_i^t$ is the clock face showing the time, and $\tilde m_i^t$ is the image where the clock face showing the time is masked.
Therefore, given preference data pairs $(m_{i}, q, y_r)$ and $(\tilde m_i^{t}, q, y_r)$, if the unmasked target $m_{t}$ is more appropriate for $q$ and $y_r$ than the masked target $\tilde m_i^{t}$, we reconstruct the loss function in equation \ref{eq:to} by adding the following conditions.

\begin{equation}
\begin{aligned}
\mathcal L_{\text{c}}= -\mathbb{E}_{(m_i,\tilde{m}_i^t,x, y_r) \sim D} [\log \sigma (u^*(m_i,\tilde{m}_i^t,x, y_r)],
\label{eq:co}
\end{aligned}
\end{equation}

With additional target conditions,  $u^*(m_i,\tilde{m}_i^t,x, y_r)$ is defined as an extension of \ref{eq:u} as follows:

\begin{equation}
u^*(m_i, \tilde{m}_i^t, x, y_r) = r(m_i, x, y_r) - r(\tilde{m}_i^t, x, y_r).
\label{eq:u*}
\end{equation}

The challenge here is to ensure that responses based on masked information from hallucination-inducing object parts are less preferred than responses that use all available information to provide a correct answer. To achieve this, our study applies masking using the bounding box of the hallucination-inducing object. This setup encourages the MLLM to utilize the object’s image in order to respond accurately.

The objective of \mymethod \ is a combination of the
target generation loss (eq.\ref{eq:to}) and target condition loss (eq.\ref{eq:co}):
\begin{equation}
\mathcal L_{TL-DPO} = \mathcal L_{t}+ \mathcal L_{c}
\end{equation}

\begin{table*}[]
\centering
\begin{tabular}{|l|llll|llll|}
\hline
                    & \multicolumn{4}{l|}{Hallucination Benchmark}                                                                                 & \multicolumn{4}{l|}{Comprehensive Benchmark}                                                                                                                                                                 \\ \hline
Method            & \multicolumn{1}{l|}{\begin{tabular}[c]{@{}l@{}}CHAIR\_s\\ ($\downarrow$)\end{tabular}} & \multicolumn{1}{l|}{\begin{tabular}[c]{@{}l@{}}CHAIR\_i\\ ($\downarrow$)\end{tabular}} & \multicolumn{1}{l|}{\begin{tabular}[c]{@{}l@{}}POPE\\ ($\uparrow$)\end{tabular}}  & {\begin{tabular}[c]{@{}l@{}}MMHal\\ ($\uparrow$)\end{tabular}} & \multicolumn{1}{l|}{\begin{tabular}[c]{@{}l@{}}SciQA-\\ IMG\end{tabular}($\uparrow$)} & \multicolumn{1}{l|}{{\begin{tabular}[c]{@{}l@{}}MM-Vet\\ ($\uparrow$)\end{tabular}}} & \multicolumn{1}{l|}{{\begin{tabular}[c]{@{}l@{}}MMBench\\ ($\uparrow$)\end{tabular}}} & \begin{tabular}[c]{@{}l@{}}LLaVA-\\ Bench($\uparrow$)\end{tabular}  \\ \hline
LLaVA-1.5           & \multicolumn{1}{l|}{66.8}          & \multicolumn{1}{l|}{12.7}         & \multicolumn{1}{l|}{85.90}          & 2.42          & \multicolumn{1}{l|}{66.8}                                                 & \multicolumn{1}{l|}{30.5}          & \multicolumn{1}{l|}{63.0}          & 63.4                                                   \\ \cline{2-9} 
+ target condition  & \multicolumn{1}{l|}{32.4}          & \multicolumn{1}{l|}{8.6}          & \multicolumn{1}{l|}{84.40}          & 2.58          & \multicolumn{1}{l|}{67.0}                                                 & \multicolumn{1}{l|}{31.3}          & \multicolumn{1}{l|}{61.2}          & 66.5                                                   \\ \cline{2-9} 
+ target generation & \multicolumn{1}{l|}{\textbf{14.6}} & \multicolumn{1}{l|}{6.1}          & \multicolumn{1}{l|}{\textbf{89.64}}          & 2.70          & \multicolumn{1}{l|}{68.2}                                                 & \multicolumn{1}{l|}{32.1}          & \multicolumn{1}{l|}{62.4}          & 68.7                                                   \\ \hline
TL-DPO (both)        & \multicolumn{1}{l|}{20.1}          & \multicolumn{1}{l|}{\textbf{5.2}} & \multicolumn{1}{l|}{86.95} & \textbf{2.72} & \multicolumn{1}{l|}{\textbf{70.3}}                                        & \multicolumn{1}{l|}{\textbf{32.2}} & \multicolumn{1}{l|}{\textbf{67.8}} & \textbf{71.2}                                          \\ \hline
\end{tabular}
\caption{\textbf{Ablation Results on \mymethod} \ with target generation or target condition removed. Both components contribute to correcting hallucinations, but the combined approach achieves the best overall performance (highlighted in bold).}
\label{tab:ablation}
\vspace{-15pt}
\end{table*}
\section{Experiment Results}
\label{experiment}

\subsection{Experiment Settings}
\subsubsection{Implementation Details}
We applied \mymethod, to LLaVA-v1.5-7B to validate the model’s performance. Model training was conducted with a batch size of 32 for 3 epochs, using a learning rate of 0.00001, a cosine learning rate scheduler, and a warm-up ratio of 0.1. The $\beta$ value for preference optimization was set to 0.1. Following previous research, we tuned the model using LoRA, setting the $\alpha$ value to 128 and the rank to 64. \mymethod, and the standard \mymethod, share the same configuration settings.

\subsubsection{Training Data}
Training Data: Based on the VG dataset (\cite{vg}), responses to given queries were generated using a baseline model. GPT-4 was then used to classify the responses as correct or incorrect, with the incorrect responses forming a hallucination dataset. Subsequently, incorrect responses were paired with the corrected answers from GPT-4 to create a preference dataset. Detailed generation methods are provided in the Appendix C.

\subsubsection{Evaluation Benchmark}
To validate performance, we used eight multimodal LLM benchmarks, focusing on hallucination issues. MMHalBench and Object HalBench assess response and object-level hallucinations, with Object HalBench measuring rates at both $CHAIR_s$ (response) and $CHAIR_i$ (object) levels. POPE evaluates preference alignment based on human feedback, while SciQA-IMG and MM-Vet assess scientific and veterinary-specific Q\&A accuracy, respectively. MMBench offers a general multimodal evaluation of factual accuracy, and LLaVA-Bench focuses on mitigating visual hallucinations in LLaVA models. Together, these benchmarks provide a comprehensive framework for evaluating MLLM performance in diverse multimodal contexts.

\begin{figure}
\centering
\includegraphics[width=0.75\linewidth]{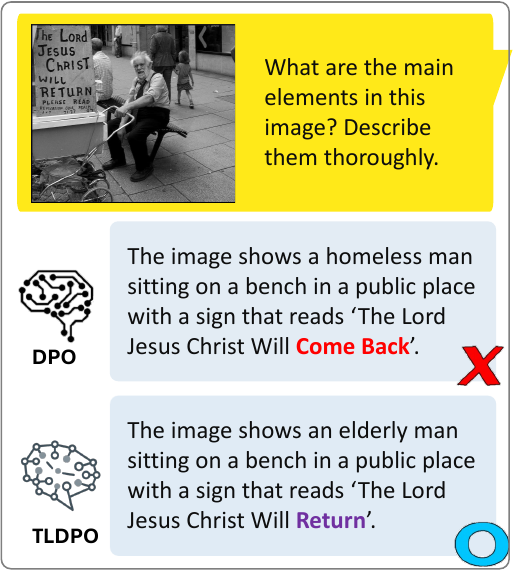}
\caption{\textbf{Qualitative Results from MMHalBench.}
An MLLM trained with standard DPO often rephrases image content into approximate meanings, leading to hallucinations resembling the correct answer (e.g., responding with “Come Back” instead of “Return” in OCR tasks). In contrast, an MLLM trained with TL-DPO gives accurate answers grounded in the actual objects in the image. Additional results are in the Appendix E.
}
\label{fig:enter-label}
\end{figure}

\subsection{Main results}
Table \ref{tab:general} summarizes the experimental results on two benchmark categories: the Hallucination Benchmark and the Comprehensive Benchmark. In the Hallucination Benchmark, our method, \mymethod, consistently achieves the lowest hallucination scores, showing significant reductions across key metrics. Specifically, \mymethod, records scores of 20.1 in $CHAIR_s$, 5.2 in $CHAIR_i$, and 2.72 in MMHal, significantly outperforming all other methods. These reductions indicate that \mymethod, effectively mitigates hallucinations related to target objects, addressing a major limitation of conventional MLLM-based models. By selectively masking or generating content centered on target object information to enhance factual consistency, \mymethod, demonstrates outstanding performance in effectively controlling hallucination-prone elements in model outputs.

Similarly, in the Comprehensive Benchmark, \mymethod, leads in overall scores, achieving 70.3 in SciQA-IMG, and the highest values in MM-Vet (32.2), MMBench (67.8), and LLaVA-Bench (71.2). These high scores confirm \mymethod’s ability to maintain balanced performance across a wide range of vision-language tasks while reducing hallucinations. By combining target learning and preference-aligned data, \mymethod, achieves these results without sacrificing performance on general metrics, demonstrating robustness and adaptability across different task settings. This contrasts with methods such as VLfeedback, Human-Preference, RLHF-V, and POVID, which struggle to achieve the same balance between accuracy and hallucination reduction. Additionally, results from HALVA and HA-DPO show further improvements in hallucination reduction but at the cost of comprehensive performance metrics.

\subsection{Generalization Across Models}
Table \ref{tab:multi} shows the performance improvements across both the Hallucination Benchmark and Comprehensive Benchmark for the LLaVA-1.5, Qwen VL Chat, LLaVA-Next, Llama3, and InternVL-2.5 models enhanced with \mymethod. For LLaVA-1.5, \mymethod, significantly reduces the $CHAIR_s$ score from 66.8 to 20.1 and the $CHAIR_i$ score from 12.7 to 5.2, while also improving metrics such as POPE and MMHal. Almost all baseline models also benefit from \mymethod, integration, showing substantial reductions in $CHAIR_s$ and $CHAIR_i$ scores and consistently improved or maintained performance on comprehensive metrics such as MM-Vet, SciQA-IMG, and MMBench. These results confirm that \mymethod, enhances robustness and generalization across a range of MLLM models, not only LLaVA, effectively reducing hallucinations while maintaining strong overall performance. However, while improvements are evident in the Hallucination Benchmark, certain metrics in the Comprehensive Benchmark show slight declines, suggesting room for further optimization in future work.

\subsubsection{Ablation study}

In this study, we propose a target learning approach that focuses on learning specific targets, adding a masking method to the target’s location on the image to enable the model to learn the target while utilizing information from the image. Table \ref{tab:ablation} presents the ablation study results for \mymethod \ 's two methods: target condition and target generation, both individually and combined. Applying only the target condition already shows substantial improvements in hallucination metrics (e.g., reducing $CHAIR_s$ from 66.8 to 32.4) while maintaining stable performance on comprehensive benchmarks. Applying target generation also enhances the results, achieving lower hallucination scores ($CHAIR_s$ 14.6) and higher comprehensive metrics (POPE 89.64). The combined \mymethod \ achieves the best performance across most benchmarks, demonstrating the effectiveness of integrating target condition and target generation for target-focused learning.

\section{Related Works}
\label{sec:related works}

\paragraph{Factuality and Hallucinations}
With recent advances in MLLMs, research on errors and hallucinations in model-generated sentences has been actively pursued. MLLMs, trained on large datasets, can produce linguistically coherent sentences across various topics, but they often present information that does not align with factual accuracy. Such hallucinations are considered particularly problematic in applications where factual accuracy is crucial, such as information retrieval, medicine, and law. Key causes of hallucinations include data bias in the training process, noise within the training data, and information gaps between pre-training and fine-tuning stages. As a result, various approaches have been proposed for factual verification in sentence generation. These approaches include evaluating model outputs through reliability scores and reducing hallucination potential by integrating with external knowledge bases.

\paragraph{Preference Alignment in MLLMs}
Aligning large models with human preferences has become a critical issue due to safety and ethical considerations in real-world applications. Preference alignment can be broadly categorized into two approaches: feedback-based alignment, which includes feedback generated by both humans and AI, and alignment through prompt guidance. Initial research on preference alignment for MLLMs has recently been conducted. Specifically, methods have been proposed to distill preferences to enhance relevance and accuracy based on MLLMs’ visual context. These methods collect human preferences to adjust model behavior by modifying hallucinated content at the segment level. While these initial results are promising, they heavily rely on traditional preference data generation processes for LLMs, where both preferred and non-preferred responses are generated, but neither is necessarily guaranteed to be accurate. In contrast, our proposed method directly corrects hallucination issues through a preference learning approach that leverages target information. This approach specifically targets hallucinations by utilizing information from the image objects that previously caused hallucinations, guiding the model to provide correct responses and effectively overcoming the limitations of existing methods.

\section{Conclusion}
\label{conclusion}
In conventional MLLM preference learning, a significant amount of information had to be acquired to correct hallucinations in image-related responses. To address this issue and mitigate object hallucinations in MLLMs, we propose \mymethod, a learning method that utilizes a dataset containing target object information. This approach uses generative data augmentation to selectively alter the correct answer phrases, generating pairs of hallucinated and correct responses along with target information for the target objects. These pairs are used to train MLLMs through the proposed \mymethod, enabling the model to accurately respond with information about the target objects and to effectively locate the correct target objects. Our extensive research demonstrates the effectiveness of \mymethod \ in mitigating various forms of object hallucinations, including hallucinations related to existence and attributes, visual illusions, and hallucinations caused by complex charts. Unlike traditional fine-tuning-based solutions, \mymethod \ effectively mitigates hallucinations across diverse vision-language tasks while maintaining or even enhancing performance in general vision-language tasks.
{
    \small
    \bibliographystyle{ieeenat_fullname}
    \bibliography{main}
}
\clearpage
\setcounter{page}{1}
\maketitlesupplementary

\section*{A. Side Effects of Preference Learning}
\label{sec: side effect} 
Here, we observe changes in the attention map to evaluate the side-effect of preference learning on mitigating hallucinations. Details in figure \ref{fig: supap}.

\section*{B. Proof of Theory} 

\label{sec:rationale}
\newtheorem{lem}{Lemma}
\begin{lem}\label{hf}
Given a reward function $r(x, y)$, assuming Assumption \ref{assum}, we can establish the equivalence between the Bradley-Terry model to the \ref{eq:BT}.
\end{lem}

\begin{proof}

We need to show that under Assumption 1, the traditional Bradley-Terry model applied to the full response is equivalent to Equation \ref{eq:BT} (our assumed equation) applied to the target chunks.
1. Traditional Bradley-Terry Model (Applied to Full Response)

$$P(y_r \succ y_h \mid x) = \sigma(r(x, y_r) - r(x, y_h))$$

Where, $y_r$ is the revised full response, $y_hl$ is the hallucinated full response, $r(x, y)$ is the reward for the full response $y$ given input $x$.

\paragraph{}\ 
We are given that $\sum r_t = \sum r$. In our case, this means:
$$r(x, y_r) = r_t(x, y_r^t), r(x, y_h) = r_t(x, y_h^t)$$

This is because Assumption \ref{assum} states that the reward difference is entirely contained within the target chunks.

\paragraph{} 
Substituting the values from step 2 into the traditional Bradley-Terry model, we get:
$$P(y_r \succ y_h \mid x) = \sigma(r_t(x, y_r^t) - r_t(x, y_h^t))$$

\paragraph{}
This is exactly the same as Equation \ref{eq:BT} (our assumed equation):

$$P(y_r^t \succ y_h^t \mid x) = \sigma(r_t(x, y_r^t) - r_t(x, y_h^t))$$

\end{proof}

\newtheorem{theo}{Theorem}
\begin{theo}
Under Assumption \ref{assum}, preference learning remains equivalent when non-target tokens are excluded. 
\end{theo}

\begin{proof} We need to show that: 
$$\mathbb{E}{y \sim \pi_\theta} [ r(x, y) ] = \mathbb{E}{y^t \sim \pi_\theta} [ r(x, y^t) ]$$

Left-hand side (Original RLHF): $\mathbb{E}{y \sim \pi_\theta} [ r(x, y) ]$ This represents the expected reward over all possible responses $y$ generated by the policy $\pi_\theta$ given the input $x$.

Right-hand side (Target-restricted RLHF): $\mathbb{E}{y^t \sim \pi_\theta} [ r(x, y^t) ]$ This represents the expected reward over all possible target chunks $y^t$ generated by the policy $\pi_\theta$ given the input $x$.

\paragraph{} 
Assumption \ref{assum} states that $\sum r_t = \sum r$. In the context of expected values, this implies:

$$r(x, y) = r(x, y^t)$$
for any response $y$ and its corresponding target chunks $y^t$. Therefore:
$$\mathbb{E}{y \sim \pi_\theta} [ r(x, y) ] = \mathbb{E}{y \sim \pi_\theta} [ r(x, y^t) ]$$

Since the reward function only considers the target chunks $y^t$ due to Assumption \ref{assum}, we can replace the expectation over all responses $y$ with the expectation over target chunks $y^t$ without changing the value:

$$\mathbb{E}{y \sim \pi_\theta} [ r(x, y^t) ] = \mathbb{E}{y^t \sim \pi_\theta} [ r(x, y^t) ]$$

Thus:
$$\mathbb{E}{y \sim \pi_\theta} [ r(x, y) ] = \mathbb{E}{y^t \sim \pi_\theta} [ r(x, y^t) ]$$

This proves the equivalence of the expected rewards.

\paragraph{}
We also need to show that the policy gradient update for target-restricted RLHF follows the same form as for the original RLHF.

Original RLHF Gradient Update:
$$\nabla_\theta \mathbb{E}_{y \sim \pi_\theta} \left[ r(x, y) - \beta D_{\text{KL}} (\pi_\theta | \pi_\text{ref}) \right]$$

Target-restricted RLHF Gradient Update:
$$\nabla_\theta \mathbb{E}_{y^t \sim \pi_\theta} \left[ r(x, y^t) - \beta D_{\text{KL}} (\pi_\theta | \pi_\text{ref}) \right]$$

Since we have already established that $\mathbb{E}{y \sim \pi_\theta} [ r(x, y) ] = \mathbb{E}{y^t \sim \pi_\theta} [ r(x, y^t) ]$, we can substitute this into the original RLHF gradient update:

\begin{equation}
\begin{aligned}
\nabla_\theta \mathbb{E}_{y \sim \pi_\theta} & [ r(x, y) - \beta D_{\text{KL}} (\pi_\theta | \pi_\text{ref}) ] = \\ & \nabla_\theta \mathbb{E}_{y^t \sim \pi_\theta} [ r(x, y^t) - \beta D_{\text{KL}} (\pi_\theta | \pi_\text{ref})]
\end{aligned}
\end{equation}

This shows that the policy gradient updates for the original RLHF and target-restricted RLHF are the same under Assumption \ref{assum}.

As previously proven, under Assumption \ref{assum}, since  $r(x, y) = r(x, y^t)$ , we can replace each response pair $(y_r, y_h)$ in the preference dataset  $\mathcal D$  with the corresponding target chunks $(y_r^t, y_h^t)$. Therefore, the objective function of the target-learning DPO is as follows:

\begin{equation}
\begin{aligned}
&L_{\text{TL-DPO}}(\theta) = \\ & -\mathbb{E}_{(x, y_r^t, y_h^t) \sim D} \Bigg [ \log \sigma \Bigg( \beta \log \frac{\pi_\theta(y_r^t|x)}{\pi_{\text{ref}}(y_r^t|x)} - \beta \log \frac{\pi_\theta(y_h^t|x)}{\pi_{\text{ref}}(y_h^t|x)} \Bigg) \Bigg ]
\end{aligned}
\end{equation}

In conclusion, we can see that the existing preference learning method including RLHF, DPO can be applied in the same way to target learning.

\end{proof}

\newtheorem{pro}{Proposition}
\begin{pro} \textbf{Efficiency comparison in target learning} 
Let  $\mathcal{H}_{pl}$  and  $\mathcal{H}_{tl}$  be the hypothesis spaces to learning methods without target($pl$) and with target($tl$), respectively. The number of samples required to achieve the same generalization error  $\epsilon$  and confidence level  $1 - \delta$  satisfies  $m_{tl} \le m_{pl}$. 
\end{pro}

\begin{proof}

 Target learning restricts the problem by focusing on a subset $y^t$ of the full output $y$. The functions in $\mathcal{H}_{tl}$ only need to discriminate based on variations within $y^t$. In contrast, functions in $\mathcal{H}_{pl}$ must accommodate variations across the entire $y$. Since $y^t$ represents a smaller, more specific part of the output space compared to $y$, the class of functions needed to model preferences over $y^t$ ($\mathcal{H}_{tl}$) is inherently less complex than the class needed for $y$ ($\mathcal{H}_{pl}$). While any preference function in $\mathcal{H}_{tl}$ can be represented within $\mathcal{H}_{pl}$ (by ignoring non-target parts), $\mathcal{H}_{pl}$ must also contain functions sensitive to variations outside $y^t$, which are explicitly excluded from consideration in $\mathcal{H}_{tl}$. Therefore, the complexity of $\mathcal{H}_{tl}$ is strictly less than that of $\mathcal{H}_{pl}$:
    \[ \text{VCD}(\mathcal{H}_{tl}) < \text{VCD}(\mathcal{H}_{pl}) \]
    The strict inequality holds because $\mathcal{H}_{pl}$ needs the capacity to model potential preference influences from non-target parts, a capacity not required by or included in $\mathcal{H}_{tl}$.

 Since the required sample complexity $m$ increases monotonically with the VC dimension for fixed $\epsilon$ and $\delta$, and we have established that $\text{VCD}(\mathcal{H}_{tl}) < \text{VCD}(\mathcal{H}_{pl})$, it follows directly that:
    \[ m_{tl} < m_{pl} \]
    Thus, target-focused preference learning is theoretically more sample-efficient than conventional preference learning for achieving the same generalization guarantees regarding the target phenomena.

\end{proof}

\section*{C. Details about dataset construction} 
\label{sec: Dataset}
This section describes the processes used to construct the dataset and the prompts employed during these processes.
The dataset construction consists of five main steps:
\begin{itemize}
    \item Step 1. Extract images, question-answer pairs associated with the images, and the bounding boxes of objects mentioned in the questions from the Visual Genome dataset.

    \item Step 2. Use a baseline model to generate responses to the queries.

    \item Step 3. Compare the model’s responses with the answers and filter out only the hallucinated responses that provide incorrect answers.

    \item Step 4. Compare the images, questions, answers, and hallucinated responses to correct the hallucinated responses into accurate answers.

    \item Step 5. Compare the hallucinated responses with the revised responses, and retain the revised positions in both the hallucinated responses and revised responses as target positions.
\end{itemize}
Finally, the dataset we constructed includes images, questions, hallucinated responses, revised responses, target positions, and bounding boxes. In Steps 3 and 4, the ChatGPT model was employed to perform the following tasks:

\begin{itemize}
    \item Prompt 1. Identifying correctness and errors in the model’s responses (in Step 3)
    
    \item Prompt 2. Correcting the incorrect model responses (in Step 4) 
\end{itemize}
The prompts used in each step are listed in Table\ref{prompt}, providing a comprehensive view of the data generation framework. Examples are also included in the Figure \ref{fig:dataset}.

\begin{figure*}
    \centering
    \includegraphics[width=0.55\linewidth]{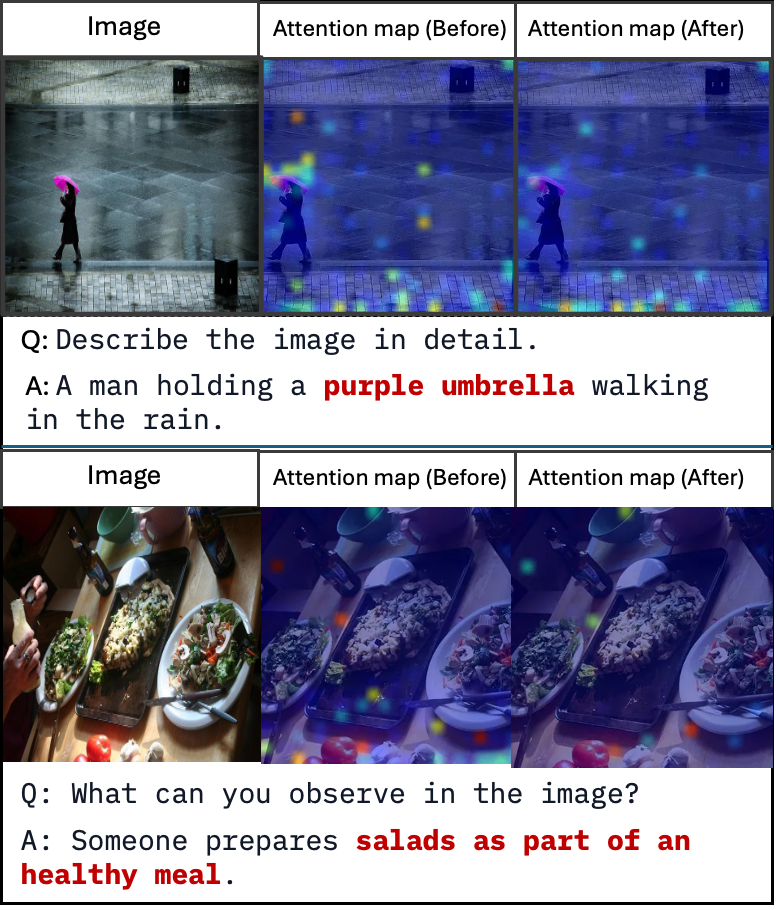}
    \caption{Changes in the Attention Map According to Preference Learning}
    \label{fig: supap}
\end{figure*}

\begin{table}[h]
    \centering
        \caption{Two types of prompts to GPT-4o (used in Step3, Step4)}
        \vspace{0.5em}
    \setlength{\tabcolsep}{10pt}
    \renewcommand{\arraystretch}{1.2}
    \begin{tabularx}{\textwidth}{X}
        \toprule
        \textbf{Identifying correctness in the model’s responses:} \\Help me evaluate the correctness of the model’s responses by comparing them to the dataset’s Question-Answer pair.\\
        *****************************************\\
            Question-answer Pair:\\
            Q: \{question\}\\
            A: \{answer\}\\
            Response:\\
            R: \{response\}\\
            Requirements:\\
            (1) Compare the Response to the Question with the Answer and determine its correctness.\\
            (2) Provide the result as either "true" or "false".\\
            *****************************************\\
            Output Format:\\
            Output: \{your answer\} \\
        \midrule
        \textbf{Prompts for revising hallucination tasks:} \\Requirements:\\
        (1) Modify parts of the given incorrect response to make it a correct response.\\
        (2) Compared to the original output, the modified response should be corrected based on the provided image and the correct answer.\\
        (3) Highlight the corrected parts by wrapping them with asterisks (e.g., corrected text).\\
        *****************************************
        Output Format:\\
          Revised answer: \{your answer\}\\
           
   \bottomrule
    \end{tabularx}
    \label{prompt}
\end{table}

\begin{figure*}
    \centering
    \includegraphics[width=0.7\linewidth]{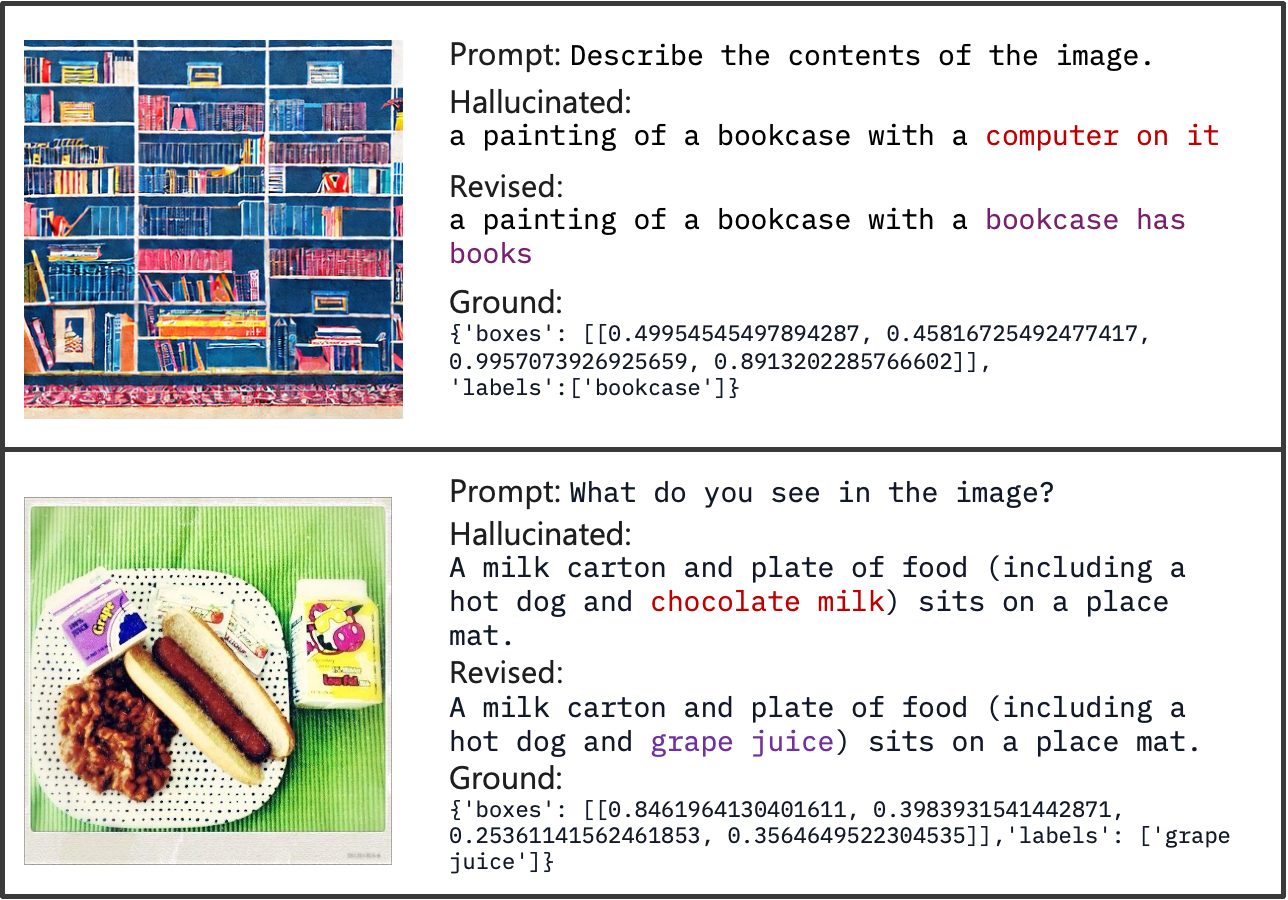}
    \caption{Examples of constructed datasets used in \mymethod}
    \label{fig:dataset}
\end{figure*}

\section*{D. Algorithms} 
The following (\ref{algorithm}) is the pseudocode for \mymethod.

\begin{algorithm*}[h]
\caption{\mymethod}\label{alg:mymethod}
\begin{algorithmic}[1]
\Require 
\begin{itemize}
    \item $\mathcal{D}$: Dataset composed of images $m$, text contexts $x$, target positions $t$, and bounding box $b$.
    \item $\pi_\theta$: Parameters of the multimodal language model (MLLM).
    \item $\pi_{\text{ref}}$: Parameters of the reference model.
    \item $\alpha$, $\beta_1$, $\beta_2$: Hyperparameters.
\end{itemize}

\State \textbf{Function Definitions:}
\begin{itemize}
    \item \Call{LabelToText}{$x$, $t$}:
        \begin{itemize}
            \item Labels non-target parts of the text context $x$ to ignore.
            \item Truncates the labels after the target position $t$.
            \item Returns the modified labels of text context $l_t(x)$.
        \end{itemize}
    \item \Call{TargetNoisyMasking}{$m$, $b$}:
        \begin{itemize}
            \item Applies a noisy mask to the image $m$ at the regions specified by bounding box $b$.
            \item Returns the modified image $\tilde m_t^r$.
        \end{itemize}
\end{itemize}

\State Generate \mymethod\ data and update $\mathcal{D}$ accordingly.
\State Initialize model parameters $\pi_\theta$.

\For{\textbf{each} epoch}
    \For{\textbf{each} $(m, x, t, b) \in \mathcal{D}$}
        \State $l_t(x) \gets$ \Call{LabelToText}{$m$, $x$, $k$, $t$}
        \Comment{Labeling and truncation for $\mathcal{L}_{\mymethod}$ calculation}
        \State $\tilde m_t \gets$ \Call{TargetNoisyMasking}{$m$, $b$}
        \Comment{Apply noisy masking to target object in image}

        \State \textbf{Compute} loss $\mathcal{L}_{\mymethod}$ using Equation (4.3)
        \State \textbf{Update} $\pi_\theta$ by minimizing $\mathcal{L}_{\mymethod}$
    \EndFor
\EndFor
\label{algorithm}
\end{algorithmic}
\end{algorithm*}

\section*{E. Qualitive Results} 
The following are the results of the qualitative analysis from the hallucination benchmark dataset in figure \ref{fig:qa}.

\begin{figure*}[b]
    \centering
    \includegraphics[width=0.8\linewidth]{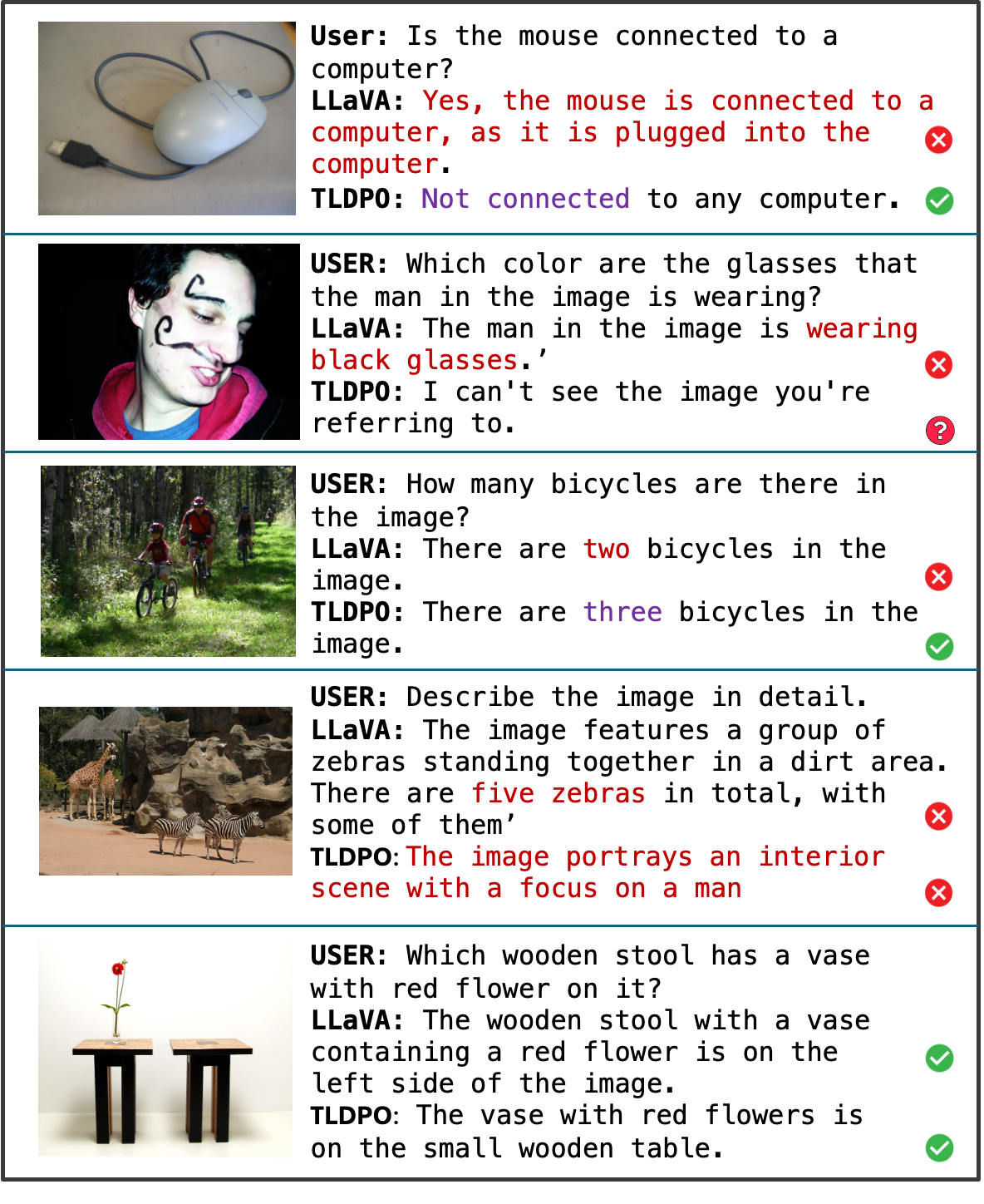}
    \caption{Qualitative Evaluation of Hallucination Correction Performance}
    \label{fig:qa}
\end{figure*}

\end{document}